\documentclass[11pt]{article}
	
	\newcommand{\blind}{0}
	
    \makeatletter
    \renewcommand\section{\@startsection {section}{1}{\z@}%
                                       {-3.5ex \@plus -1ex \@minus -.2ex}%
                                       {2.3ex \@plus.2ex}%
                                       {\normalfont\fontfamily{phv}\fontsize{16}{19}\bfseries}}
    \renewcommand\subsection{\@startsection{subsection}{2}{\z@}%
                                         {-3.25ex\@plus -1ex \@minus -.2ex}%
                                         {1.5ex \@plus .2ex}%
                                         {\normalfont\fontfamily{phv}\fontsize{14}{17}\bfseries}}
    \renewcommand\subsubsection{\@startsection{subsubsection}{3}{\z@}%
                                        {-3.25ex\@plus -1ex \@minus -.2ex}%
                                         {1.5ex \@plus .2ex}%
                                         {\normalfont\normalsize\fontfamily{phv}\fontsize{14}{17}\selectfont}}
    \makeatother
	
	\usepackage{amsmath}
    \usepackage{amssymb}
    \usepackage{amsfonts}
    \usepackage{amsthm}
    \usepackage{multirow}
    \usepackage{booktabs}
    \usepackage{algorithm}
    \usepackage{algorithmic}
	\usepackage{graphicx}
	\usepackage{enumerate}
    \usepackage{subcaption}
	\usepackage{xcolor}
	\usepackage{natbib} 
	\usepackage{url} 

    \theoremstyle{plain}
    \newtheorem{proposition}{Proposition}
	
\begin{document}
		
		\def\spacingset#1{\renewcommand{\baselinestretch}%
			{#1}\small\normalsize} \spacingset{1}
		
		\if0\blind
		{
			\title{A Predict-then-Correct Loop Based on Few-Shot Continuous Contextual Bandit for Demand Forecasting}
			\author{Zhiwei Lei$^a$, Benedict Jun Ma$^a$ and Ilya Jackson$^b$ \\
			$^a$Thrust of Intelligent Transportation, \\ Hong Kong University of Science and Technology (Guangzhou), China; \\ $^b$Center for Transportation and Logistics, \\ Massachusetts Institute of Technology, USA; \\}
			\date{}
			\maketitle
		} \fi
		
		\if1\blind
		{
            \title{A Predict-then-Correct Loop Based on Few-Shot Continuous Contextual Bandit for Demand Forecasting}
			\author{Author information is purposely removed for double-blind review}
			
\bigskip
			\bigskip
			\bigskip
			\begin{center}
				{\textbf{\Large A Predict-then-Correct Loop Based on Few-Shot Continuous Contextual Bandit for Demand Forecasting}}
			\end{center}
			\medskip
		} \fi
		\bigskip
		
	\begin{abstract}
Retail demand forecasting remains difficult when demand shifts faster than static forecasting models can be retrained, especially in early demand cycles where newly observed labels are sparse. To address this, this study aims to improve adaptive retail forecasting by proposing a predict-then-correct (PtC) framework that retains a first-stage machine learning (ML) forecast and applies a few-shot continuous contextual bandit correction policy with similar-SKUs augmentation and top-$p$ masked updating. Across Walmart retail data and an exclusive beverage dataset, PtC delivers statistically significant reductions in MAPE, MAE, and RMSE across stable \& high volume, stable \& low volume, and erratic \& intermittent demand patterns, improves average RMSE by 9.52\% over the ML-only baseline in the ablation study, and yields lower inventory costs than base-stock, proximal policy optimization, and soft actor-critic policies under the tested lead-time settings. These findings show that online forecast correction can bridge offline demand learning and real-time retail decision-making by adapting to sparse feedback without fully retraining the base forecasting model.
	\end{abstract}
			
	\noindent%

	\spacingset{1.5} 

\section{Introduction} 
\label{s:intro}
Retail supply chains are high-dimensional, uncertain, and complex decision systems that tightly couple demand forecasting, inventory control, and multi-stage logistics operations \citep{cohen2022demand}. This system-level complexity is driven primarily by three interrelated dimensions: stochastic demand, network structure, and dynamic decision-making. First, customer demand is non-stationary and volatile, often affected by exogenous market conditions, pricing mechanisms, and promotional activities \citep{fahimnia2025service}. Second, managing millions of stock-keeping units (SKUs) across multi-echelon distribution networks introduces intricate product-level relationships, such as substitution effects, complementarity, and cross-item resource sharing \citep{ma2016demand}. Third, key operational decisions, including inventory allocation \citep{khouja2019early,kim2024multiagent}, dynamic pricing \citep{zhong2020gametheoretic}, and robust procurement \citep{wagner2015robust,li2024supply}, exhibit complex time-dependent characteristics \citep{huang2019forecasting}. Consequently, inaccurate demand forecasts can propagate beyond the forecasting stage and distort downstream operational decisions, leading to severe stockouts or costly backlogs that degrade overall supply chain performance \citep{cohen2022demand}.

With the rapid development of artificial intelligence (AI) and enterprise big data infrastructures, modern supply chains are actively evolving into data-driven cyber-physical systems that achieve real-time perception and dynamic closed-loop optimization \citep{punia2022predictive, mandania2023dynamic}. In this context, demand forecasting constitutes a fundamental analytical function for such intelligent supply chain systems. Industry evidence suggests that AI-driven forecasting frameworks can reduce prediction errors by 20\% to 50\%, lower holding costs by 5\% to 10\%, and reduce lost sales by up to 65\% \citep{mckinsey2022ai}. Methodologically, this development has accelerated a shift in forecasting research from classical parametric time-series models toward supervised machine learning (ML) architectures, such as long short-term memory (LSTM) networks \citep{hochreiter1997long}, gradient boosting machines (GBM) \citep{friedman2001greedy}, XGBoost \citep{chen2016xgboost}, and deep learning (DL) architectures such as transformer-based variants \citep{zhang2024intermittent, li2024enhancing}. Due to the abilities of nonlinear mapping, these data-driven approaches have demonstrated decisive algorithmic superiority over traditional statistical benchmarks in large-scale forecasting \citep{makridakis2022m5, lei2024pooling}.

Despite these advances, offline-trained ML models remain limited when deployed in rapidly changing and non-stationary market environments. During a new demand cycle or a product launch period, market dynamics often suffer from sharp structural breaks compared to historical regimes, inducing acute data sparsity and prominent ``cold-start" challenges \citep{petropoulos2022forecasting}. Because traditional supervised pipelines rely on rigid, batch-oriented training and require massive data pools for expensive offline retraining, they cannot seamlessly ingest real-time demand feedback. In practice, the correction of forecast bias therefore remains largely reactive \citep{wang2024reinforcement}. Once a base model is trained and deployed, it operates as an open-loop predictor that lacks the capacity for real-time calibration against short-term operational fluctuations and concurrent temporal shifts. This reveals a critical methodological void: the absence of a systematic, sample-efficient mechanism engineered to explicitly \textit{learn to correct} baseline prediction errors on the fly.

To bridge this gap, this study proposes a novel \textbf{predict-then-correct} (PtC) loop framework designed for adaptive demand calibration. Instead of treating forecast-error correction as a passive label-fitting task, we reformulate real-time forecast adjustment as a contextual bandit (CB) problem. In the proposed framework, an offline ML model first generates baseline predictions, after which an online few-shot continuous contextual bandit (FSCCB) policy dynamically fine-tunes these predictions by selecting bounded, continuous corrective actions from streaming market contexts based on immediate reward feedback. The core advantage of this architecture lies in its decoupled, closed-loop design: it preserves the structural, historical knowledge captured by the offline base model, while enabling the system with real-time, step-by-step responsiveness as market observations gradually unfold. To resolve data scarcity and the stability-plasticity dilemma inherent in early-cycle launches, the framework further integrates a feature-similar sequence data augmentation strategy and a sparse parameter update mechanism \citep{mazumder2021few} that updates only highly sensitive policy parameters while freezing the primary structural network. Ultimately, this framework shifts the operational forecasting paradigm from static prediction to autonomous, self-calibrating intelligent control.

The main contributions of this paper are summarized as follows:

(1) We propose a two-stage PtC framework that couples AI/ML forecasting with a CB correction model. This transforms static, open-loop demand forecasting into an adaptive, closed-loop system, significantly enhancing prediction responsiveness to market fluctuations.

(2) To tackle real-time data scarcity, we develop a few-shot continuous updating strategy. A base model is pretrained on historical data for initialization, and then dynamically fine-tuned via streaming data augmentation and a selective parameter update mechanism, enabling rapid adaptation to new demand patterns while preventing overfitting and catastrophic forgetting.

The remainder of the paper is organized as follows. Related works are reviewed in the next section. Section \ref{sec:third} describes the details of the two-stage framework. The experimental results are discussed in Section \ref{sec:experiments}. Section \ref{sec:inventory_ptc} discusses the application of PtC in the inventory management system. Conclusions are presented in Section \ref{sec:conclusion}.

\section{Literature review} 
\label{s:review}
\subsection{Traditional statistical methods for demand forecasting}

Classical demand forecasting research first developed around extrapolative time-series models. Naive, moving-average, exponential-smoothing, and ARIMA-family models remain widely used because they are transparent, computationally efficient, and effective when demand contains stable level, trend, or seasonal components \citep{taylor2003exponential,nikolopoulos2011adida,hyndman2018forecasting}. Exponential-smoothing models place greater weight on recent observations and therefore provide a simple way to adapt to gradual level or trend changes, while ARIMA-type models use autoregressive and moving-average structures to capture serial dependence after differencing. These models are especially attractive in operational environments because they are easy to implement and require limited feature engineering. Retail applications extended these ideas to volatile and skewed demand settings: early work on retailer demand forecasting identified promotions, competitor actions, weather, and holidays as important drivers \citep{geurts1986forecasting}, while later studies developed interval forecasts, hybrid SARIMAX quantitative methods, and intermittent-demand aggregation strategies for demand series with irregular arrivals or heavy-tailed errors \citep{taylor2007forecasting,arunraj2015hybrid,nikolopoulos2011adida}. The strength of this stream is its clarity: the source of each forecast is usually traceable to level, seasonality, trend, or a small set of lagged terms. Its limitation is equally clear. Once the functional form has been specified, the model has limited flexibility to represent nonlinear promotion effects, changing consumer response, cross-SKU substitution, or sudden shifts in demand regimes. Thus, classical statistical models provide reliable baselines, but they are less suited to settings where the main forecasting challenge is rapid adaptation rather than stable extrapolation.

A second stream uses causal, multivariate, and econometric models to incorporate explanatory factors into demand forecasts. In retail systems, demand is rarely driven by time alone. All prices, discounts, display activities, holidays, competitive information, search intensity, weather, social-media signals, and product attributes can alter realized sales. Therefore, operations and forecasting studies have examined how external information can be integrated into forecast pipelines and how judgmental or promotional adjustments affect forecast quality \citep{fildes2008forecasting,trapero2013analysis,steinker2017weather,boone2018google,cui2018operational,fildes2022retail}. In fast-moving consumer goods and retail settings at the SKU-level, variable selection and structured econometric models have been used to address high-dimensional promotional variables, competitive effects, and time-varying marketing impacts \citep{ma2016demand,huang2014value,huang2019forecasting,ye2024forecasting}. Hierarchical Bayesian and finite-mixture extensions further examine whether store-level heterogeneity improves forecast and elasticity estimation \citep{andrews2008estimating}. These studies make an important contribution by linking prediction to interpretable business mechanisms: instead of treating demand as an autonomous time series, they show how observed operational levers can improve forecast accuracy and managerial understanding. However, this gain in explanatory richness creates a practical tradeoff. Causal and multivariate models require well-measured covariates, stable relationships between covariates and demand, and enough observations to estimate the effects reliably. When many promotional, price, calendar, and competitive variables are included, the model can face sparse event coverage and a heavy variable-selection burden \citep{trapero2015identification}. These difficulties are particularly acute for new products, short sales windows, and low-volume SKUs, where the data needed to estimate a rich explanatory model may not yet exist.

Large-scale empirical evidence also warns that model complexity does not automatically translate into better forecasts. The M-competitions showed that statistically sophisticated methods may fail to outperform simpler alternatives on broad forecasting benchmarks \citep{makridakis1982accuracy,makridakis1993m2,makridakis2000m3}, and related evidence argues that excessive model complexity can reduce accuracy, increase error opportunities, and weaken managerial usability \citep{green2015simple}. This finding is important for retail supply chains because forecasting is not an isolated statistical exercise. Forecasts are inputs to replenishment, allocation, pricing, and inventory-control decisions, so a method that is marginally more accurate but difficult to maintain may still be unattractive in practice. Studies in fast-fashion inventory and online-retail analytics illustrate that demand forecasts create value only when they are timely and operationally usable \citep{caro2010inventory,ferreira2016analytics}. A forecast produced after a long retraining cycle may be too late for ordering decisions; a model that requires extensive manual feature redesign may be too costly for thousands of SKUs; and a method that performs well on average may still fail when early-cycle observations are sparse. 

\subsection{Machine learning and deep learning models}

ML and DL methods have expanded the forecasting toolkit by allowing nonlinear mappings, high-dimensional covariates, and large-scale pattern learning. Review studies report a broad adoption of ML tools in supply-chain demand forecasting, supplier selection, order allocation, and related operational prediction tasks \citep{aamer2021data,ingle2021demand,islam2021machine,feizabadi2022machine,malviya2024systematic}. In retail demand prediction, richer information sources and supervised-learning methods have enabled models to use weather, promotion, search, social-media, product-attributes and price information on a scale that is difficult for traditional statistical models \citep{steinker2017weather,boone2018google,cui2018operational,lei2023new,ferreira2016analytics,alley2023pricing}. The 2020 M5 Accuracy competition further illustrates this shift: leading methods were ML-based and achieved strong accuracy improvements over traditional statistical benchmarks on large-scale hierarchical Walmart sales data \citep{makridakis2022m5}. This line of work shows that data-intensive models can represent complex retail patterns, pool signals across products, and exploit high-dimensional covariates. It also changes the role of forecasting in supply chains. Forecasting is no longer only an extrapolation task based on a single SKU history; it becomes a pattern-learning problem that uses cross-product, temporal, promotional, and contextual signals to infer future demand.

Within this broad family, different model classes address different sources of complexity. Tree-based ensemble models, such as gradient boosting frameworks, are effective when forecasting is based on tabular covariates, nonlinear interactions, and heterogeneous feature effects. Neural networks and deep sequence models capture nonlinear relationships and temporal dependencies. Early evidence of ANN, CNN-based time-series models, LSTM architectures, attention mechanisms, and Transformer variants have been used to improve retail or time-series forecasting \citep{adya1998effective,alon2001forecasting,aburto2007improved,zhang2019short,graves2005framewise,mamdouh2024improving,gao2023adversarial,zhang2024intermittent,li2024enhancing}. These models are attractive because they can learn representations rather than requiring researchers to specify all interactions manually. In the retail setting, this is especially useful when demand is jointly affected by price, promotion, calendar, assortment, and historical sales signals. At the same time, deep models can be data-hungry, sensitive to distribution shifts, and difficult to interpret. Their empirical success often depends on the availability of large training histories and on whether future conditions resemble the historical regimes used for training. When demand changes after deployment, simply relying on the offline model may be insufficient, even if the model was strong on the original test set.

Other ML streams complement deep sequence forecasting by addressing uncertainty, dependency, and robustness. Gaussian process and neural network applications in agricultural, commodity, real-estate, and financial price series demonstrate the value of flexible nonlinear learning and uncertainty-aware prediction in related time-series environments \citep{xu2021corn,xu2023steel,jin2025csi300}. Graph and causal structure methods address interconnected markets and contemporaneous dependencies, which are relevant to retail systems where products may substitute, complement, or share demand drivers \citep{jin2025jiangsu}. Ensemble and composite forecasts improve robustness by combining multiple base learners and reducing dependence on a single model specification \citep{guo2025bayesian}. These approaches broaden the modeling landscape, but they also highlight a recurring limitation. However, the literature has developed many powerful first-stage predictors, but the second-stage question of real-time forecast correction remains less developed.

The need for adaptation has motivated online learning, reinforcement learning, and continual-learning approaches in operations and forecasting. Online-learning studies emphasize updating models as new data arrive, while reinforcement-learning studies show how operational policies can learn from repeated interaction with uncertain environments \citep{petropoulos2022forecasting,wang2024reinforcement,boute2022deep,gijsbrechts2022can}. These studies are important because they move beyond purely static prediction and recognize that decisions and observations unfold sequentially. However, a direct reinforcement-learning formulation may be excessive for demand correction when the immediate goal is not to learn an entire inventory policy but to adjust a baseline forecast before downstream decisions are made. Similarly, online retraining can be costly and unstable when only a few new observations are available. Continual-learning and few-shot studies further show that unrestricted parameter updates can cause overfitting or catastrophic forgetting under limited new data \citep{kirkpatrick2017overcoming,tian2024survey,bethune2025scaling,mazumder2021few}. This creates a stability-plasticity dilemma: a model must adapt enough to correct new bias, but not so much that it destroys the useful structure learned from previous demand cycles.

The closest unresolved issue is therefore not whether advanced forecasting models can be accurate offline, but how their forecasts should be adapted online after new demand feedback is observed. Most demand forecasting pipelines still generate static predictions, select among static models, or require costly retraining when demand changes. They rarely treat each forecast as a decision that can be immediately corrected using contextual feedback. This distinction is central to our study. The PtC framework does not attempt to replace statistical, machine-learning, or deep-learning forecasters. Instead, it uses them as the first stage and adds a CB-based correction policy as the second stage. The CB formulation is appropriate because each correction action receives immediate feedback through the change in forecast error, which makes the learning signal local, interpretable, and sample efficient. The few-shot continuous updating design further addresses the limited-data setting by updating only selected parameters through the Top-$p$ mechanism, thus preserving prior correction knowledge while allowing local adaptation.

\section{Solution methods}
\label{sec:third}
\subsection{AI/ML prediction and data augmentation}
\label{sec:PDA}
Figure \ref{fig:frw} illustrates the overall framework of the proposed PtC method. In the first stage, using average demand, coefficient of variation (CoV), and intermittency as clustering features, SKUs will be divided into four categories: stable \& high volume, stable \& low volume, erratic \& intermittent, and lumpy \citep{ma2025data}. Each category is then associated with an ML algorithm for prediction. For the target SKU, the first step is to identify its category and then generate a baseline forecast using the corresponding ML model. The resulting baseline forecast is used as input to the second stage. Augmentation data selection within each cluster is provided in Appendix \ref{app:augmentation_selection}.

\begin{figure}[htbp]
    \centering
    \includegraphics[width=0.8\textwidth]{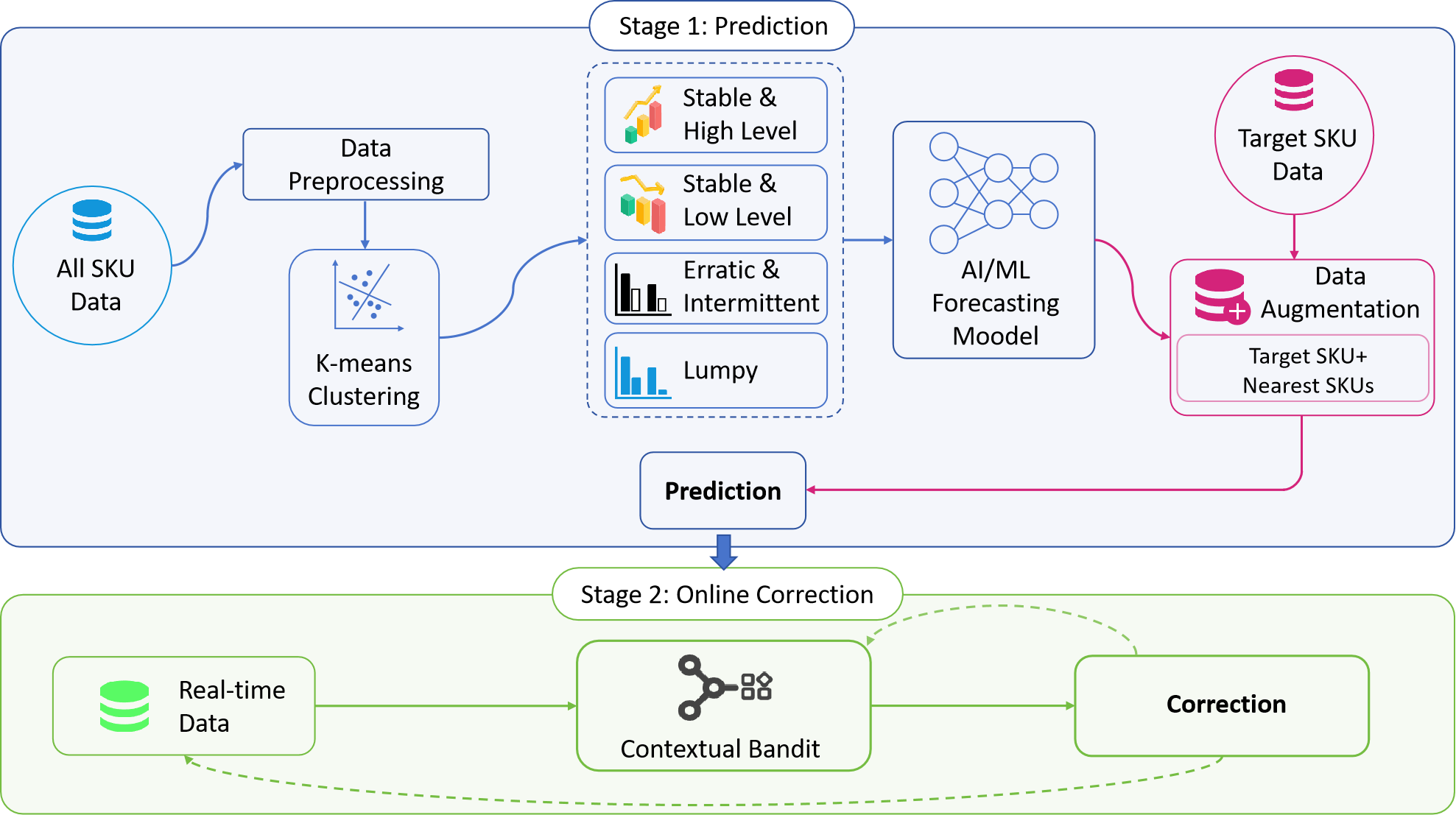}
    \caption{Total framework of the proposed method.}
    \label{fig:frw}
\end{figure}

In the second stage, we adopt the CB as the correction model to adjust the ML-predicted demand by real-time observations within a sliding window. To support timely prediction and avoid delaying operational decisions, only five or seven days of demand information can be collected; otherwise, the correction loses its practical relevance. In this setting, real-time demand information from a single SKU is insufficient to train an effective correction policy and may cause severe overfitting or catastrophic forgetting. Therefore, it is necessary to design a data augmentation strategy to expand the training set and alleviate these risks. 

For the target SKU, there are often numerous SKUs within the same cluster that share similar characteristics. Thus, we first rank candidate SKUs by their Euclidean distance to the target SKU and then select the $k$ nearest SKUs to augment the training set. Meanwhile, an upper distance threshold $d_{\mathrm{is}}$ is set to avoid selecting SKUs that are too dissimilar from the target SKU. If fewer than $k$ SKUs satisfy this threshold, we apply same-trend perturbations to both historical demand and ML-predicted demand until a pre-specified minimum training-set size is reached. Clustering, scaling, neighbor selection, and augmentation parameters are estimated using training-period information only to prevent test-period leakage.

\subsection{FSCCB}
\subsubsection{State, action, and reward design}
With full observation enabled, the observed state, or context, of the PtC system at time step $t$ is denoted as:
\begin{equation}
    \label{eq:state}
    s_t =
    \left\{
    g^d_r, g^d_x, g^d_e, g^d_{hr}, g^d_{hx}, g^d_s,
    g^t_d, g^t_r, g^t_w, g^t_h, g^t_s
    \right\}.
\end{equation}

The description of each component is shown in Table \ref{tab:state_representation}. Specifically, $g^d_r$ represents the historical real demand, $g^d_x$ represents the historical ML-predicted demand, and $g^d_e$ represents the historical prediction error. In addition, $g^d_{hr}$ and $g^d_{hx}$ denote the mean historical real demand and the mean historical ML-predicted demand, respectively, while $g^d_s$ measures the historical slope of the real demand. Time-related features include the current time step, the remaining time before the end of the cycle, and binary indicators for weekend, holiday and shopping festival effects.

\begin{table}[htbp]
\centering
\caption{State representation components.}
\begin{tabular}{lll}
\hline
\textbf{Type} & \textbf{Description} & \textbf{Notation} \\
\hline
                 & Historical real demand & ${g}_{r}^d$ \\
                 & Historical ML-predicted demand & ${g}_{x}^d$ \\
Historical data  & Historical ML error & ${g}_{e}^d$ \\
                 & Mean historical real demand & ${g}_{hr}^d$ \\
                 & Mean historical ML-predicted demand & ${g}_{hx}^d$ \\
                 & Historical slope of real demand & ${g}_{s}^d$ \\
\hline
                      & Current time step & ${g}_{d}^t$ \\
                      & Remaining time before the end of the cycle & ${g}_{r}^t$ \\
Time characteristics  & Is weekend or not & ${g}_{w}^t$ \\
                      & Is holiday or not & ${g}_{h}^t$ \\
                      & Is shopping festival or not & ${g}_{s}^t$ \\
\hline
\end{tabular}
\label{tab:state_representation}
\end{table}

In the CB framework, $s_t$ is used as the contextual information available before making the correction decision. The policy network uses this context to determine how the original ML prediction should be adjusted.

In this study, the action is defined as the continuous adjustment factor applied to the
ML-predicted demand. Given the observed context $s_t$, the policy network outputs the
mean and standard deviation of a Gaussian correction distribution, denoted by
$\mu_{\theta}(s_t)$ and $\sigma_{\theta}(s_t)$, respectively. During training and online
adaptation, the correction action is sampled as
\begin{equation}
    \label{eq:stochastic_action}
    \tilde{a}_t
    \sim
    \mathcal{N}
    \left(
    \mu_{\theta}(s_t),
    \sigma_{\theta}^{2}(s_t)
    \right),
    \qquad
    a_t
    =
    \operatorname{clip}
    \left(
    \tilde{a}_t,-1,2
    \right).
\end{equation}

Here, $a_t$ denotes the sampled relative correction ratio. A negative value indicates that the original ML prediction should be reduced, while a positive value indicates that the original ML prediction should be increased. The action range is limited to $[-1,2]$ to prevent the correction model from generating unreasonable adjustment values. The stochastic sampling in Eq. (\ref{eq:stochastic_action}) provides local exploration over possible upward and downward forecast adjustments.

Given the original ML-predicted demand $\hat{y}^{ML}_t$, the corrected demand is calculated as:
\begin{equation}
    \label{eq:inverse}
    \hat{y}^{CB}_t
    =
    \hat{y}^{ML}_t(1+a_t^*),
\end{equation}
where $\hat{y}^{CB}_t$ denotes the demand value corrected by the CB model.

Therefore, the CB policy does not directly predict the final demand. Instead, it learns how to adjust the baseline ML prediction according to the observed context.

The reward is defined as the relative degree of improvement after correction. At each time step $t$, the prediction error of the original ML model is calculated as:
\begin{equation}
    \label{eq:ML_err}
    e_{ML,t}
    =
    \left|
    \hat{y}^{ML}_t-y_t
    \right|,
\end{equation}
where $y_t$ denotes the actual demand.

Similarly, the prediction error after the CB correction is calculated as:
\begin{equation}
    \label{eq:Corr_err}
    e_{CB,t}
    =
    \left|
    \hat{y}^{CB}_t-y_t
    \right|.
\end{equation}

To evaluate whether the correction action improves the original ML prediction, the reward function is designed as:
\begin{equation}
    \label{eq:reward}
    r_t
    =
    \tau
    \left(
    e_{ML,t}-e_{CB,t}
    \right),
\end{equation}
where $\tau$ is a scaling factor used to normalize the reward into a suitable numerical range.

According to Eq. (\ref{eq:reward}), a positive reward indicates that the CB correction reduces the prediction error compared with the original ML prediction. A negative reward means that the correction worsens the prediction result. Therefore, maximizing the reward is equivalent to learning a correction policy that improves the forecasting accuracy of the baseline ML model.

\subsubsection{CB for forecast correction}

In this study, the forecast correction problem is formulated as a continuous contextual bandit problem. The proposed PtC framework first uses the baseline ML model to generate an initial demand prediction. Then, the CB model samples a continuous correction action according to the stochastic policy $\pi_{\theta}(\cdot \mid s_t)$, the main steps are provided in Figure \ref{fig:bandit}.

The correction process can be represented as:
$s_t \rightarrow a_t^* \rightarrow r_t,$
where $s_t$ denotes the observed context, $a_t^*$ denotes the correction action, and $r_t$ denotes the reward measuring the improvement after correction. At each time step $t$, the correction model first observes the context $s_t$, which is constructed from recent demand, baseline ML forecasts, historical forecast errors, and time-related features. The policy then samples a continuous correction action $a_t$ from $\pi_{\theta}(\cdot \mid s_t)$ and applies it to the baseline forecast. This stochastic action-selection mechanism enables exploration around the current policy mean and allows the model to test whether increasing, decreasing, or retaining the baseline forecast is more beneficial under the current context.

\begin{figure}[htbp]
    \centering
    \includegraphics[width=\textwidth]{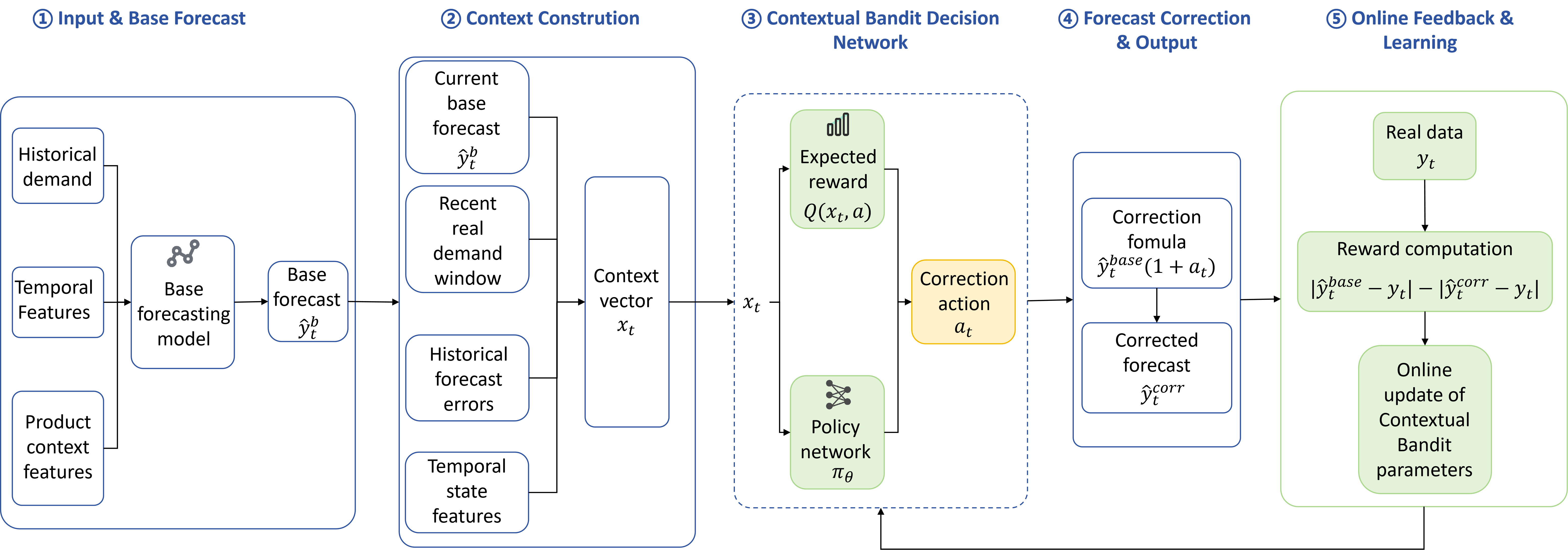}
    \caption{Main steps of the second stage in PtC.}
    \label{fig:bandit}
\end{figure}

The bandit feedback structure is essential in this formulation. Before the actual demand $y_t$ is observed, the model only has access to $s_t$ and must commit to one correction action $a_t$. After $y_t$ is realized, the reward $r_t$ is computed by comparing the baseline forecast error and the corrected forecast error. The model uses the reward of the selected action to update the policy, while no supervised target action is imposed. The newly observed tuple: $(s_t,a_t,r_t,\log \pi_{\theta}(a_t\mid s_t))$ is then added to the sliding-window feedback set and used for future policy updates.
Therefore, PtC learns from bandit feedback by increasing the probability of correction actions that receive positive rewards and decreasing the probability of actions that worsen the baseline forecast.

\subsubsection{Learning Objective of the CB Correction Model}

Let the chronological feedback set be:
\begin{equation}
    \mathcal{D}
    =
    \left\{
    (s_t,a_t^*,r_t)
    \right\}_{t=1}^{T},
\end{equation}
where each tuple contains the context, the correction action, and the realized reward. The goal is to learn a policy $\pi_{\theta}$ that maps the observed context to a continuous correction ratio:

\begin{equation}
    \label{eq:policy_mapping}
    a_t^*
    =
    \pi_{\theta}(s_t).
\end{equation}

Given the first-stage forecast $\hat{y}^{ML}_t$, the corrected forecast induced by this policy is:
\begin{equation}
    \label{eq:cb_corrected_demand}
    \hat{y}^{CB}_t
    =
    \hat{y}^{ML}_t
    \left(1+\pi_{\theta}(s_t)\right).
\end{equation}
Thus, the policy network does not replace the first-stage forecasting model. It only learns a context-dependent adjustment factor for the existing forecast.

At the population level, the CB correction policy is trained to maximize the expected reward obtained after correction:
\begin{equation}
    \label{eq:cb_objective}
    \pi_{\theta}^{*}
    =
    \arg\max_{\pi_{\theta}}
    \mathbb{E}_{s_t \sim \mathcal{D}}
    \left[
    r_t(s_t,\pi_{\theta}(s_t))
    \right].
\end{equation}
This objective is appropriate for forecast correction because each action is evaluated by its immediate effect on the current prediction error. Unlike a full reinforcement-learning formulation, no long-horizon value function is required; the correction model only needs to learn whether the current context calls for increasing, decreasing or retaining the baseline forecast.

In empirical training, the expected objective is approximated by the average reward over the feedback set:
\begin{equation}
    \label{eq:cb_train_objective}
    \theta^*
    =
    \arg\max_{\theta}
    \frac{1}{T}
    \sum_{t=1}^{T}
    r_t.
\end{equation}

Substituting the reward definition from Eq. (\ref{eq:reward}) into Eq. (\ref{eq:cb_train_objective}) gives:

\begin{equation}
    \label{eq:cb_final_objective}
    \theta^*
    =
    \arg\max_{\theta}
    \frac{1}{T}
    \sum_{t=1}^{T}
    \tau
    \left(
    \left|\hat{y}^{ML}_t-y_t\right|
    -
    \left|
    \hat{y}^{ML}_t
    \left(1+\pi_{\theta}(s_t)\right)
    -
    y_t
    \right|
    \right).
\end{equation}

The first error term in Eq. (\ref{eq:cb_final_objective}) is determined by the baseline ML forecast and does not depend on $\theta$. Therefore, maximizing the reward is equivalent to minimizing the corrected prediction error:
\begin{equation}
    \label{eq:cb_loss}
    \mathcal{L}(\theta)
    =
    \frac{1}{T}
    \sum_{t=1}^{T}
    \left|
    \hat{y}^{ML}_t
    \left(1+\pi_{\theta}(s_t)\right)
    -
    y_t
    \right|.
\end{equation}
This equivalence connects the reward-based contextual bandit formulation with the standard forecasting objective: a high reward means that the correction policy reduces the absolute error of the first-stage forecast.

\begin{proposition}
\label{prop:zero-correction}
Let $L_{PtC}^{*}$ and $L_{ML}$ denote the optimal corrected and ML-only expected losses induced by any nonnegative forecasting loss $\ell$. If $0\in\mathcal{A}$ and $\pi_0(s)=0\in\Pi$, then $L_{PtC}^{*}\le L_{ML}$.
\end{proposition}

Proposition \ref{prop:zero-correction} establishes a hypothesis-class containment result at the population optimum. Since the zero-correction policy $\pi_0(s)=0$ is included in the PtC policy class, PtC contains the ML-only forecast as a special case, and its optimal population loss is no larger than that of the ML-only forecast. When the correction action is zero for every context, the corrected forecast exactly reduces to the original ML forecast. This result should be interpreted as an expressiveness and containment property of the PtC hypothesis class, rather than as a finite-sample performance guarantee: it does not imply that the estimated PtC model will necessarily outperform the ML-only baseline with limited samples or at every online decision step. A proof of Proposition \ref{prop:zero-correction} is provided in Appendix \ref{app:proof1}.

In implementation, historical demand cycles are replayed in chronological order to initialize the policy, and the same feedback structure is used during online deployment. This design preserves the decision timing of the CB formulation: the policy observes $s_t$, selects $a_t^*$, receives $r_t$ after demand is realized and then updates itself for future correction decisions.

\subsubsection{Top-$p$ update}

In a new demand cycle, customer demand is gradually observed over time. The CB correction model must update its parameters to incorporate new information from the demand environment. However, in such a data-scarce scenario, frequent parameter updates may lead to overfitting or catastrophic forgetting of previously learned knowledge \citep{kirkpatrick2017overcoming, tian2024survey, bethune2025scaling}. To mitigate this risk, we adopt a parameter update rule based on the proportion of parameters.

We adopt a top-$p$ masked update rule to control the degree of online adaptation. Specifically, we select the top-$p$ proportion of parameters with the smallest absolute values and update only these parameters, while the remaining parameters are frozen. Small-magnitude parameter selection is a heuristic intended to preserve highly expressed pretrained parameters; its effectiveness must therefore be established empirically. Intuitively, parameters with larger absolute values are treated as more strongly expressed components of the pretrained correction policy, whereas smaller-magnitude parameters provide a limited adaptation subspace for learning new-cycle feedback.

To implement this rule, we first sort the parameters in each layer by their absolute values. The threshold $\delta^k$ is then set to the largest magnitude within the smallest $p\%$ parameters of layer $k$. Parameters whose absolute values are no larger than $\delta^k$ are selected for online updating, while the remaining parameters are kept fixed.

After the threshold is determined, we adopt a binary mask matrix in equation (\ref{eqa:mask}) to record which parameters need to be updated as follows:
\begin{equation}
\label{eqa:mask}
M(\theta_j^k)=
\begin{cases}
1, & |\theta_j^k| \le \delta^k \\
0, & \text{otherwise}
\end{cases}
\end{equation}
where $\theta_j^k$ is the $j$-th parameter in layer $k$. If $M(\theta_j^k) = 1$, the parameter $\theta_j^k$ is selected to be updated. The shift of the update rule can be found in equation (\ref{eqa:update}).
\begin{equation}
\label{eqa:update}
\theta_j^k \leftarrow \theta_j^k - \eta\, g_j^k
\quad \rightarrow \quad
\theta_j^k \leftarrow \theta_j^k - \eta\, M(\theta_j^k)\, g_j^k
\end{equation}
where $\eta$ represents the learning rate. In our research, we set the learning rate to a very small value to enable fine-tuning updates, thereby avoiding disruption of the knowledge learned before \citep{zhang2023slca, tao2020few}. 

\begin{proposition}
\label{prop:top-p-tradeoff}
Let $\Theta_p=\{\theta^0+\Delta:\operatorname{supp}(\Delta)\subseteq S_p\}$ and $S_{p_1}\subseteq S_{p_2}$ for $p_1\le p_2$.
Then $G_{ad}(p)=\inf_{\theta\in\Theta_p}\mathcal{L}_{new}(\theta)-\inf_{\theta\in\Theta_1}\mathcal{L}_{new}(\theta)$ is non-increasing in $p$.
If $|g_j^u|\le G$, then $\|\theta_p^U-\theta^0\|_2\le \eta U G\sqrt{\lceil pd\rceil}$, so the forgetting-risk bound is non-decreasing in $p$.
\end{proposition}

Proposition \ref{prop:top-p-tradeoff} characterizes the adaptation--stability tradeoff induced by the top-$p$ update rule. The first part states that increasing $p$ enlarges the set of trainable parameters. Since $\Theta_{p_1}\subseteq\Theta_{p_2}$ when $p_1\le p_2$, a larger update set cannot increase the best achievable loss on the new demand cycle. Therefore, the adaptation gap $G_{ad}(p)$ is non-increasing in $p$. This explains why updating too few parameters may limit the ability of the correction policy to adapt to a new demand pattern. proof of Proposition \ref{prop:top-p-tradeoff} is provided in Appendix \ref{app:proof2}.

\begin{proposition}
\label{prop:continuous-action-convergence}
Let $\mathcal{J}_{\epsilon}(\theta)=\mathbb{E}_{\mathcal{D},a_t\sim\pi_{\theta}}[\rho_{\epsilon}(\hat y_t^{ML}(1+a_t)-y_t)]$ be lower bounded and $L$-smooth.
Under $\theta^{u+1}=\theta^u-\eta\,m_u\odot\nabla\mathcal{J}_{\epsilon}(\theta^u)$ with $0<\eta\le1/L$,
the PtC correction policy satisfies $\mathcal{J}_{\epsilon}(\theta^{u+1})\le\mathcal{J}_{\epsilon}(\theta^u)$ for $u=0,1,\ldots$.
\end{proposition}

Proposition \ref{prop:continuous-action-convergence} provides a stability guarantee for the masked update when the correction objective is represented by a smooth surrogate loss. The absolute forecasting error is not differentiable at zero, so $\rho_{\epsilon}(\cdot)$ is introduced as a smooth approximation of the correction error. If this surrogate objective is lower bounded and $L$-smooth, then the standard descent lemma implies that a masked gradient step with step size $0<\eta\le 1/L$ does not increase the objective value. A proof of Proposition \ref{prop:continuous-action-convergence} is provided in Appendix \ref{app:proof3}.

Pseudo-code of top-$p$ masked update is described in Appendix \ref{app:algorithm}.

\section{Experiment} \label{sec:experiments}
\subsection{Experimental design}
\subsubsection{Datasets and implementation}
All algorithms were implemented in Python 3.11 (64-bit). The experiments were conducted on a workstation equipped with an Intel Core i7-14700KF CPU @ 5.6 GHz, 32 GB DDR5 RAM, and an NVIDIA RTX 4070 Ti GPU.

We evaluate the proposed PtC framework using the Walmart sales data released through the M5 Accuracy competition in 2020 \citep{makridakis2022m5}. Specifically, we use 6,875 SKUs from Walmart stores ``CA\_1" to ``CA\_3", covering more than 1,500 daily observations.

The beverage dataset is an external industrial benchmark collected from a major North American beverage manufacturer. It contains 156 weeks of weekly shipment records 49 bottling plants and over 12,000 distribution nodes. 

The lumpy category is defined in Section 3 for completeness of SKU segmentation, but it is omitted from the later experiments because lumpy demand has long zero-demand periods and irregular spikes, making statistical and AI/ML models largely ineffective; therefore, we recommend a qualitative CPFR-based approach instead of reporting model-comparison results.

Details of dataset splitting can be found in Appendix \ref{app:dataset_splitting}, and detailed descriptions of the parameter values of the set are provided in Appendix \ref{app:hyperparameters}.

\subsubsection{Evaluation metrics}
We evaluate the precision of the forecast by mean absolute percentage error (MAPE), mean absolute error (MAE), and root mean squared error (RMSE). For a test sequence with actual demand $y_t$ and forecast $\hat{y}_t$, the metrics are defined as:
\begin{equation}
MAPE=\frac{100}{T}\sum_{t=1}^{T}
\left|
\frac{y_t-\hat{y}_t}{\max(|y_t|,\epsilon)}
\right|,
\end{equation}
\begin{equation}
MAE=\frac{1}{T}\sum_{t=1}^{T}|y_t-\hat{y}_t|,
\quad
RMSE=\sqrt{\frac{1}{T}\sum_{t=1}^{T}(y_t-\hat{y}_t)^2}.
\end{equation}
where $\epsilon$ is a small positive constant used to avoid division by zero. 

MAPE provides a scale-free percentage error, MAE reflects the average absolute deviation, and RMSE emphasizes large forecast errors.

For the demand-category comparison, we report the relative error reduction:
\begin{equation}
\label{eq:relative_reduction}
Reduction(M,q)
=
\frac{E_q^{ref}-E_q^{M}}{E_q^{ref}}\times 100\%,
\end{equation}
where $E_q^{M}$ denotes the error of method $M$ under metric $q$, and
$E_q^{\mathrm{ref}}$ denotes the corresponding error of the ML-only reference forecast.
A positive value indicates an improvement over the reference forecast, whereas a negative
value indicates a deterioration in forecasting performance.

\subsubsection{Benchmark methods}

ETS \citep{gardner1985exponential} is an exponential smoothing method that recursively updates level, trend, and seasonal components. A representative additive ETS forecast is given by
\begin{equation}
\hat{y}*{t+h}^{ETS}=\ell_t+h b_t+s*{t+h-m},
\end{equation}
where $\ell_t$, $b_t$, and $s_t$ denote the level, and seasonal states, respectively, and $m$ is the seasonal period.

ARIMA is a classical linear time-series model that captures autocorrelation through autoregressive and moving-average terms after differencing. The model can be written as
\begin{equation}
\phi(B)(1-B)^d y_t=c+\theta(B)\varepsilon_t,
\end{equation}
where $B$ is the backshift operator, $d$ is the differencing order, and $\varepsilon_t$ is the innovation term.

LightGBM \citep{ke2017lightgbm} is a gradient boosting decision tree model that builds an additive ensemble of regression trees. Its prediction can be expressed as
\begin{equation}
\hat{y}*t=\sum*{k=1}^{K} f_k(x_t), \qquad f_k \in \mathcal{F},
\end{equation}
where $x_t$ is the input feature vector and $f_k(\cdot)$ denotes the $k$-th regression tree.

XGBoost \citep{chen2016xgboost} is another boosted-tree model that learns additive regression trees by minimizing a regularized objective:
\begin{equation}
\mathcal{L}^{(k)}=\sum_{t=1}^{T}l\left(y_t,\hat{y}_t^{(k-1)}+f_k(x_t)\right)+\Omega(f_k),
\end{equation}
where $l(\cdot)$ is the loss function and $\Omega(f_k)$ penalizes tree complexity.

N-BEATS is a deep neural forecasting model composed of stacked fully connected residual blocks. Each block generates a backcast for residual updating and a forecast, while the final prediction is obtained by summing the block-level forecasts.

TiDE is a dense encoder--decoder forecasting model that maps historical demand and covariates into multi-step forecasts.

TFT is an attention-based deep forecasting model that integrates static features, time-varying covariates, gating mechanisms, and temporal attention for multi-horizon forecasting.

\subsection{Results}

\subsubsection{Numerical experiment on the Walmart dataset}
Table \ref{tab:walmart_relative_reduction} reports the average relative error reduction on the Walmart retail dataset across three distinct demand categories. The results are calculated based on Eq. (\ref{eq:relative_reduction}), where positive values indicate a performance improvement over the LightGBM ML-only reference forecast, and negative values denote an algorithmic decrease.

\begin{table*}[htbp]
\centering
\caption{Comparison result of Average relative error reduction on the Walmart dataset.}
\label{tab:walmart_relative_reduction}
\small
\setlength{\tabcolsep}{4pt}
\begin{tabular}{lccccccccc}
\toprule
\multirow{2}{*}{Method}
& \multicolumn{3}{c}{Stable \& Low Volume}
& \multicolumn{3}{c}{Stable \& High Volume}
& \multicolumn{3}{c}{Erratic \& Intermittent} \\
\cmidrule(lr){2-4}
\cmidrule(lr){5-7}
\cmidrule(lr){8-10}
& MAPE & MAE & RMSE
& MAPE & MAE & RMSE
& MAPE & MAE & RMSE \\
\midrule

ETS
& 0.05 & 4.92 & 1.83
& -9.66 & -1.10 & -2.08
& -14.24 & 3.82 & 1.18 \\

N-BEATS
& 7.31 & 5.23 & 1.42
& -5.36 & 5.40 & 1.30
& -26.85 & 5.19 & -1.60 \\

XGBoost
& -12.54 & 1.92 & -0.81
& -1.84 & 0.49 & 0.47
& -8.78 & -1.24 & -13.75 \\

ARIMA
& 1.03 & 4.68 & 1.90
& -9.31 & -4.17 & -9.32
& -14.77 & 3.67 & 1.09 \\

TiDE
& 11.57 & 10.98 & 1.95
& -5.82 & 5.17 & 0.86
& -21.97 & 4.38 & -1.29 \\

TFT
& 6.67 & 10.78 & 0.86
& -10.09 & -0.43 & -5.38
& -25.38 & 4.81 & -1.68 \\

PtC
& \textbf{12.16}\textsuperscript{***}
& \textbf{11.28}\textsuperscript{***}
& \textbf{6.87}\textsuperscript{***}
& \textbf{6.74}\textsuperscript{***}
& \textbf{9.95}\textsuperscript{***}
& \textbf{7.67}\textsuperscript{***}
& \textbf{5.68}\textsuperscript{***}
& \textbf{8.94}\textsuperscript{***}
& \textbf{3.70}\textsuperscript{***} \\

\bottomrule
\end{tabular}

\vspace{0.25em}
\begin{minipage}{0.95\textwidth}
\footnotesize
Note: Positive values indicate improvement relative to the ML-only reference forecast, while negative values indicate performance decrease.
* $p<0.10$, ** $p<0.05$, *** $p<0.01$.
\end{minipage}
\end{table*}

Across all demand patterns, the proposed PtC framework consistently shows superior adaptability compared to both traditional statistical methods and DL architectures. Notably, PtC is the only method evaluated that achieves positive relative error reduction across all three metrics in every reported setting. Furthermore, all observed improvements are statistically significant at the $p < 0.01$ level.

The empirical results expose the structural vulnerabilities of static baseline models when subjected to the highly stochastic nature of retail supply chains. Deep forecasting architectures, such as N-BEATS, TiDE, and TFT, exhibit moderate success under ``Stable \& Low Volume" conditions but degrade substantially when applied to ``Erratic \& Intermittent" demand. Specifically, N-BEATS shows a -26.85\% decrease in MAPE for intermittent items. This underscores a limitation of pure data-driven models: their susceptibility to overfitting the noisy, sparse signals inherent in intermittent demand, which subsequently induces severe predictive bias. Conversely, traditional time-series methods like ETS and ARIMA struggle significantly with ``Stable \& High Volume" series. This is likely driven by their constrained parametric forms, which fail to capture the complex, non-linear dynamics characterizing high-throughput retail operations.

In contrast, the performance of the PtC architecture achieves average error reductions between 5.68\% and 8.94\% even in the highly challenging erratic category, validating its underlying mathematical design. Rather than overriding the historical distributions captured by the base ML model, the FSCCB correction layer dynamically fine-tunes the forecast via online reward feedback. This decoupled architecture effectively bridges the methodological gap between offline historical generalization and online correction.

\subsubsection{Numerical experiment on the beverage dataset}
To better validate the performance of the proposed correction mechanism, we conduct a supplementary numerical experiment on a beverage distribution dataset. Table \ref{tab:beverage_relative_reduction} details the average relative error reduction under the previously established evaluation framework.

\begin{table*}[htbp]
\centering
\caption{Comparison result of Average relative error reduction on the beverage dataset.}
\label{tab:beverage_relative_reduction}
\small
\setlength{\tabcolsep}{4pt}
\begin{tabular}{lccccccccc}
\toprule
\multirow{2}{*}{Method}
& \multicolumn{3}{c}{Stable \& Low Volume}
& \multicolumn{3}{c}{Stable \& High Volume}
& \multicolumn{3}{c}{Erratic \& Intermittent} \\
\cmidrule(lr){2-4}
\cmidrule(lr){5-7}
\cmidrule(lr){8-10}
& MAPE & MAE & RMSE
& MAPE & MAE & RMSE
& MAPE & MAE & RMSE \\
\midrule

LightGBM
& 6.18 & 5.05 & 6.12
& -- & -- & --
& 3.91 & -1.01 & 2.93 \\

ETS
& -- & -- & --
& 7.82 & 2.89 & -3.67
& -- & -- & -- \\

N-BEATS
& 5.44 & 6.25 & 5.45
& 10.88 & -5.57 & -11.48
& -10.39 & 4.01 & 2.20 \\

XGBoost
& -5.94 & 4.71 & -5.59
& -6.24 & -7.62 & -8.46
& 0.79 & 0.64 & 2.79 \\

ARIMA
& 6.80 & -5.01 & 5.37
& 7.85 & -18.07 & -21.99
& 6.27 & 4.99 & 1.72 \\

TiDE
& 6.91 & 5.43 & 5.58
& 11.19 & -3.46 & -6.14
& 1.24 & -12.92 & -0.04 \\

TFT
& 5.49 & 6.76 & -6.00
& 8.19 & 5.88 & 2.73
& -10.33 & 7.06 & 2.27 \\

PtC
& \textbf{13.64}\textsuperscript{***}
& \textbf{7.85}\textsuperscript{***}
& \textbf{7.09}\textsuperscript{***}
& \textbf{14.64}\textsuperscript{***}
& \textbf{14.85}\textsuperscript{***}
& \textbf{13.40}\textsuperscript{***}
& \textbf{8.92}\textsuperscript{***}
& \textbf{12.37}\textsuperscript{***}
& \textbf{4.48}\textsuperscript{***} \\

\bottomrule
\end{tabular}

\vspace{0.25em}
\begin{minipage}{0.95\textwidth}
\footnotesize
Note: Positive values indicate improvement relative to the ML-only reference forecast, while negative values indicate a performance decrease.
* $p<0.10$, ** $p<0.05$, *** $p<0.01$.
\end{minipage}
\end{table*}

In this data environment, traditional methods like ARIMA and tree-based models like XGBoost exhibit severe performance decrease in the ``Stable \& High Volume" category.  Furthermore, DL models such as N-BEATS, TFT, yield highly inconsistent results across different metrics within identical demand categories. Such metric instability indicates that these base models are highly sensitive to shifting data distributions and lack the continuous learning capability required to mitigate real-time prediction errors under structural volatility.

Conversely, PtC maintains substantial, balanced, and statistically significant ($p < 0.01$) improvements among all conditions. The error reductions peak in the ``Stable \& High Volume" category, achieving 14.64\%, 14.85\%, and 13.40\% error reduction for MAPE, MAE and RMSE, respectively. By retaining the base ML model's structural understanding of baseline demand while deploying an online CB to correct for micro-level volatility, PtC shifts the forecasting paradigm from a static, open-loop prediction task to an adaptive, closed-loop correction system.

\subsection{Sensitivity analysis}
We further analyze the sensitivity of the top-$p$ parameter update proportion and the depth of the correction network.

Figure \ref{fig:PSA} compares the sensitivity of the update proportion $p$ on the North American beverage company dataset and the Walmart retail dataset. The two curves show that updating too few parameters limits the performance of the correction model, while updating too many parameters increases the risk of overfitting or forgetting previously learned demand patterns. On the beverage-company dataset, the improvement rises rapidly from a small update proportion and reaches its peak $p=0.16$, after which performance declines steadily. On the Walmart dataset, the improvement is already significant under small update proportions and reaches its highest value $p=0.10$; further increasing $p$ leads to a gradual decrease in improvement. This cross-dataset evidence suggests that FSCCB benefits from selective adaptation rather than full-network updating.

\begin{figure}[htbp]
    \centering
    \begin{subfigure}[t]{0.48\textwidth}
        \centering
        \includegraphics[width=\linewidth]{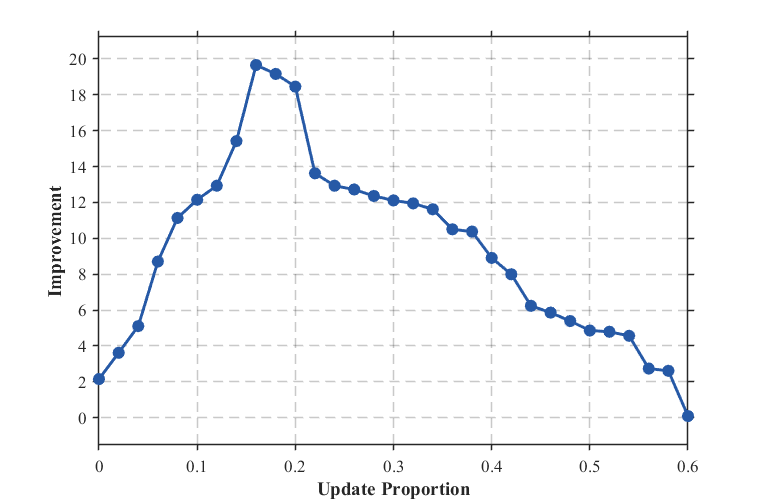}
        \caption{North American beverage company dataset}
        \label{fig:PSA_beverage}
    \end{subfigure}
    \hfill
    \begin{subfigure}[t]{0.46\textwidth}
        \centering
        \includegraphics[width=\linewidth]{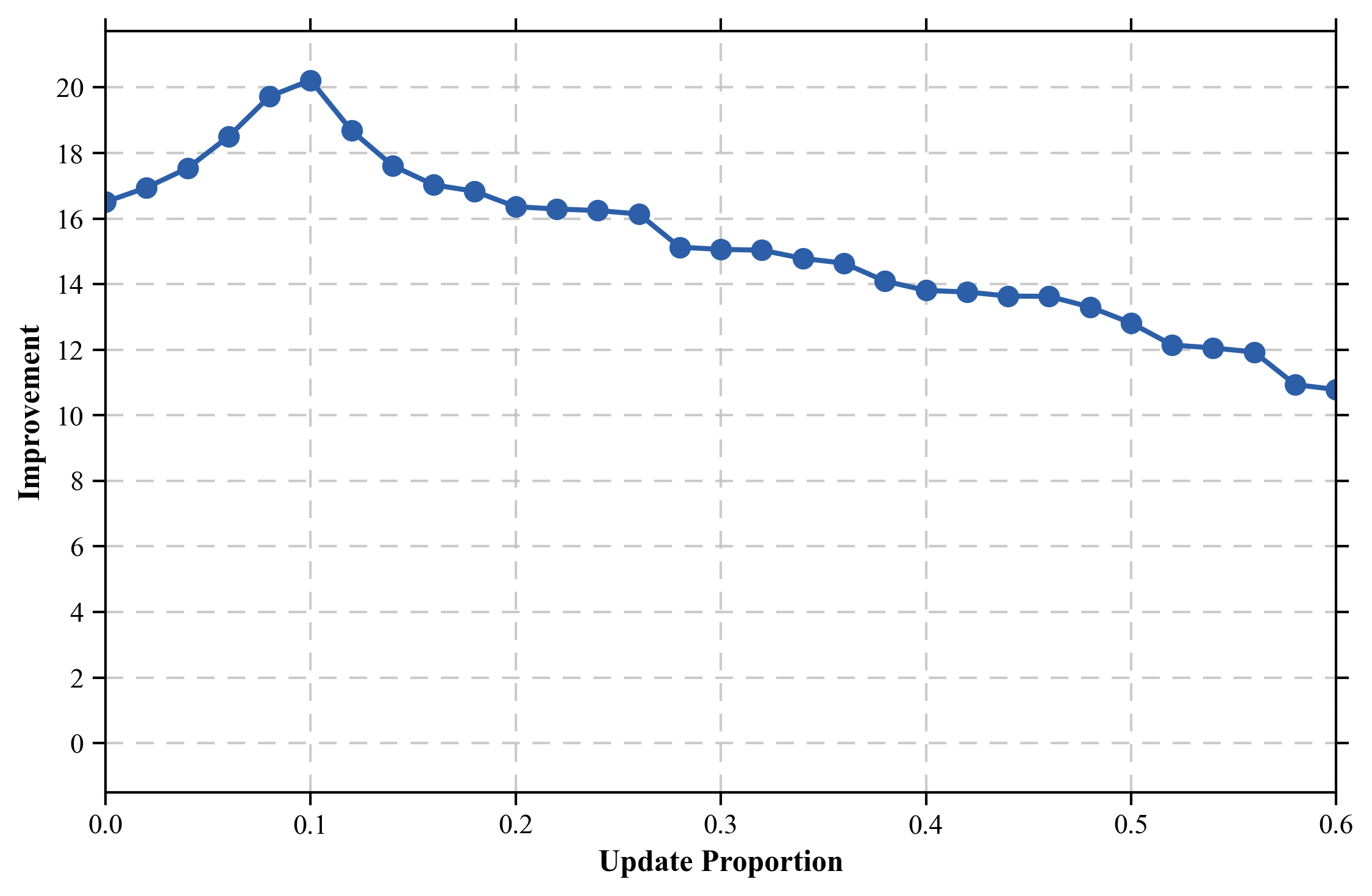}
        \caption{Walmart retail dataset}
        \label{fig:PSA_walmart}
    \end{subfigure}
    \caption{Parameter sensitivity analysis of the update proportion $p$ across two datasets.}
    \label{fig:PSA}
\end{figure}

Figure \ref{fig:LSA} reports the sensitivity to network depth. A shallow correction network has limited representation capacity and may fail to capture nonlinear relationships between context features and correction actions. However, an overly deep correction network increases the number of trainable parameters and becomes less stable under few-shot updates. The observed pattern supports the use of a moderate-depth policy network for the CB correction model.
\begin{figure}[htbp]
    \centering
    \includegraphics[width=0.5\textwidth]{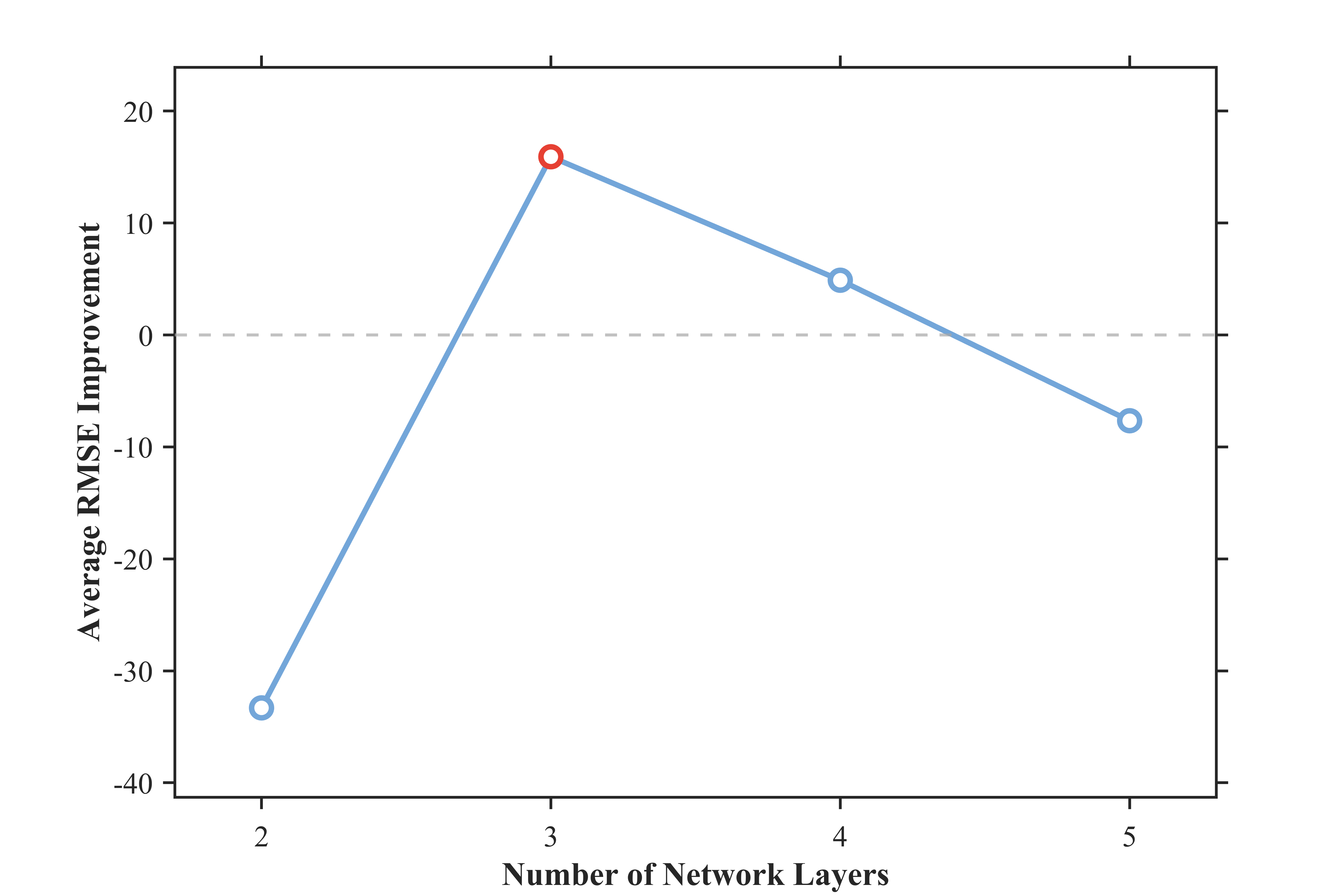}
    \caption{Network depth sensitivity analysis.}
    \label{fig:LSA}
\end{figure}

\subsection{Ablation study}
To isolate the contribution of the CB correction layer and the few-shot selective update rule, we conduct an ablation study using N-BEATS as the first-stage ML forecast. We compare three variants as follows:

(1) ML-only: Forecasting relies on the N-BEATS prediction without second-stage correction;

(2) ML+CB: A CB correction model is applied, but all correction-network parameters are updated during adaptation;

(3) ML+FSCCB: FSCCB is applied, where only the selected top-$p$ parameters are updated during adaptation.

The reported Gap is the relative RMSE reduction compared with the ML-only baseline. A positive Gap indicates that the correction module improves the N-BEATS forecast.

Table \ref{tab:ab_avg} reports the average ablation results of the proposed few-shot continuous CB method. Compared with the original N-BEATS model, the standard N-BEATS-CB only achieves a marginal RMSE reduction of 0.34\%. Moreover, although its average RMSE is slightly lower than that of N-BEATS, its average MAE becomes higher. This indicates that simply introducing a CB-based correction network does not provide a sufficiently stable improvement under limited observations from a new demand cycle.

\begin{table}[htbp]
    \centering
    \small
    \caption{Average prediction performance comparison of N-BEATS, N-BEATS-CB, and PtC.}
    \label{tab:ab_avg}
    \setlength{\tabcolsep}{5pt}
    \begin{tabular}{cc ccc ccc}
        \toprule
        \multicolumn{2}{c}{N-BEATS} 
        & \multicolumn{3}{c}{N-BEATS-CB} 
        & \multicolumn{3}{c}{PtC} \\
        \cmidrule(lr){1-2} \cmidrule(lr){3-5} \cmidrule(lr){6-8}
        RMSE & MAE 
        & RMSE & MAE & Gap 
        & RMSE & MAE & Gap \\
        \midrule
        11.4086 & 8.7993 & 11.3903 & 8.8363 & 0.34\% & \textbf{10.3372} & \textbf{7.8193} & \textbf{9.52\%} \\
        \bottomrule
    \end{tabular}
\end{table}

By contrast, PtC achieves a much larger average improvement. It reduces the average RMSE from 11.4086 to 10.3372 and the average MAE from 8.7993 to 7.8193, corresponding to an average RMSE Gap of 9.52\%. This result shows that the proposed FSCCB method can more effectively adapt the forecasting model to a new demand cycle.

The improvement mainly comes from the few-shot selective update mechanism. Instead of updating all parameters of the correction network, FSCCB updates only a small subset of low-magnitude parameters while freezing the remaining parameters. This strategy helps preserve the general correction knowledge learned from historical demand cycles and prevents the model from overfitting short-term noise in the new cycle. Therefore, the ablation results confirm that FSCCB provides a more stable and effective correction mechanism than the standard CB model, making the proposed PtC framework better suited for demand forecasting under data scarcity.

\section{PtC in Inventory Management System}
\label{sec:inventory_ptc}
\subsection{Model Formulation in Inventory Management System}
The forecasting experiments show that PtC improves prediction accuracy, but the operational value of a forecast depends on how it affects downstream decisions. We therefore embed the corrected forecasts into a periodic-review inventory management system. At the beginning of period $t$, the retailer observes the inventory position $IP_t$, outstanding replenishment orders, and the current demand context. A replenishment policy then chooses an order quantity $q_t\geq 0$. Orders arrive after a deterministic lead time $L$, and demand $y_t$ is realized at the end of the period.

Let $I_t$ denote the net inventory after demand is realized, where positive values represent inventory and negative values represent unmet demand. The inventory transition is written as:
\begin{equation}
\label{eq:inventory_balance}
I_{t+1}=I_t+q_{t-L}-y_t,
\end{equation}
where $q_{t-L}$ is the order placed $L$ periods earlier. The single-period inventory cost is:
\begin{equation}
\label{eq:inventory_cost}
C_t=h[I_{t+1}]^{+}+b[-I_{t+1}]^{+},
\end{equation}
where $h$ is the unit holding-cost coefficient, $b$ is the unit shortage-cost coefficient, and $[x]^+=\max\{x,0\}$. The objective is to minimize the average cost over the evaluation horizon:
\begin{equation}
\label{eq:average_inventory_cost}
\bar C=\frac{1}{T}\sum_{t=1}^{T}C_t.
\end{equation}

PtC enters this system through the demand signal used for replenishment. The baseline ML forecast $\hat y_t^{ML}$ is first corrected by the learned contextual-bandit action $a_t$, producing:
\begin{equation}
\label{eq:ptc_inventory_forecast}
\hat y_t^{PtC}=\hat y_t^{ML}(1+a_t).
\end{equation}
The corrected forecast is then used to estimate lead-time demand. For a lead time $L$, the PtC-based replenishment target can be expressed as:
\begin{equation}
\label{eq:ptc_inventory_target}
\hat D_{t,L}^{PtC}=\sum_{j=0}^{L-1}\hat y_{t+j}^{PtC},
\end{equation}
and the corresponding order quantity is:
\begin{equation}
\label{eq:ptc_inventory_order}
q_t^{PtC}=\max\{0,\hat D_{t,L}^{PtC}-IP_t\}.
\end{equation}
This formulation keeps the inventory decision simple while allowing the order quantity to respond to online forecast corrections.

\subsection{Baseline Policies and Experiment Settings}
We compare PtC with three inventory-control baselines: base-stock (BS), proximal policy optimization (PPO), and soft actor-critic (SAC). BS is a classical order-up-to policy that replenishes inventory toward a target level based on the estimated lead-time demand. It provides a transparent operational benchmark for forecast-driven inventory control.

PPO and SAC represent deep reinforcement learning policies for inventory management, a setting in which the policy learns order decisions from repeated interaction with the inventory simulator \citep{boute2022deep, gijsbrechts2022can}. PPO is an on-policy actor-critic algorithm that updates the policy through a clipped surrogate objective, which improves training stability by limiting overly large policy updates. SAC is an off-policy actor-critic algorithm that maximizes both expected return and policy entropy, encouraging exploration in continuous-control problems \citep{haarnoja2018soft}. In this experiment, both RL baselines observe the same inventory state information and are trained to minimize the cumulative inventory cost.

All policies are evaluated under the same demand sequences, cost parameters, and lead-time settings. We vary the lead time from $L=2$ to $L=4$ to test whether the policies remain robust when the replenishment decision must anticipate demand further into the future. Inventory cost is reported in millions, and lower values indicate better downstream operational performance.

\subsection{Numerical Results}
Figure \ref{fig:inventory_cost} reports the inventory-cost comparison across the tested lead-time settings. PtC achieves the lowest inventory cost across the reported settings, indicating that the forecast improvements observed in the earlier experiments translate into downstream cost reductions. The result is important because a lower forecast error does not automatically guarantee a lower inventory cost; the correction must also improve the timing and order quantities of replenishment under the uncertainty of lead time.

\begin{figure}[htbp]
    \centering
    \includegraphics[width=0.95\textwidth]{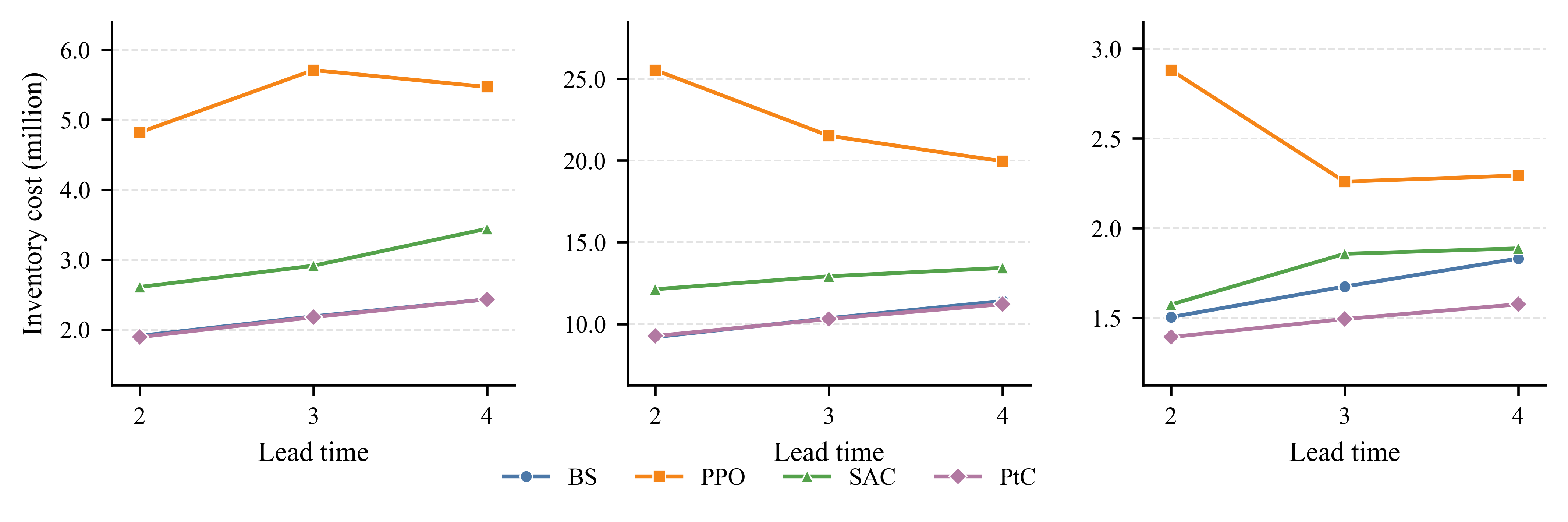}
    \caption{Inventory-cost comparison of BS, PPO, SAC and PtC under different lead times.}
    \label{fig:inventory_cost}
\end{figure}

The comparison also shows different failure modes across baselines. PPO yields the highest inventory costs in the reported settings, suggesting that the learned policy is less stable for this demand-correction task. SAC performs better than PPO but remains above PtC, indicating that a generic continuous-control policy may still struggle when the main uncertainty comes from demand forecasting errors. BS is competitive in some settings because it directly links replenishment to estimated lead-time demand, but it does not adapt the forecast itself when real-time demand feedback reveals bias. PtC combines these two advantages: it preserves a simple forecast-driven replenishment structure and improves the input forecast through online correction. 

As the lead time increases, inventory decisions become more sensitive to forecast bias because orders must cover a longer demand interval. PtC maintains the lowest cost under the tested lead times, suggesting that correcting the forecast before the replenishment decision can reduce both excess inventory and shortage risk. These results provide downstream evidence for the practical value of the PtC loop in inventory management systems.

\section{Conclusion}\label{sec:conclusion}
This paper studies adaptive demand forecasting in retail supply chains, where replenishment, pricing, purchasing, and inventory decisions require timely forecasts but demand patterns can shift before static forecasting models are fully retrained. To address this problem, we propose a PtC framework that preserves the base forecast generated by an offline machine-learning model and then applies a few-shot CB correction policy to adjust the forecasts as new feedback arrives. The main finding is that forecast correction can be treated as a sequential learning problem: instead of replacing the original forecaster, a lightweight correction layer can improve predictive reliability and downstream decision quality while retaining the structure learned from historical demand.

This study contributes to the forecasting and retail operations literature by filling a gap between static forecasting and full model retraining. Existing forecasting models often produce predictions as fixed outputs, while many operational decision models assume that demand information is already available in a usable form. PtC connects these two views by modeling the post-forecast stage as an adaptive correction process, thereby making demand forecasting part of a continuous learning loop. The similar-SKU augmentation and top-$p$ masked update further clarify how a correction policy can balance responsiveness to new observations with the preservation of previously learned demand regularities. In practice, this design can be embedded into existing forecast pipelines with limited disruption, helping retailers update forecasts for early-cycle, sparse, low-volume, or intermittent items without rebuilding the entire forecasting system.

The strength of this study lies in its modular design and its direct connection between forecast accuracy and operational usefulness. Because PtC is placed after the base forecaster, it can work as an adaptive layer rather than as a replacement for existing forecasting models, which makes the framework easier to interpret and deploy. At the same time, the current study has several limitations. The correction policy is built at the SKU level and does not yet explicitly model substitution, complementarity, shelf-space competition, or capacity coupling across products. The downstream decision setting is also simplified, and the robustness of the correction range and top-$p$ update needs further examination under more diverse retail conditions.

Future research can extend this work in several directions. First, multi-product correction policies can be developed by incorporating graph, hierarchical, or attention-based structures to capture dependencies across SKUs. Second, forecast correction can be integrated more closely with replenishment, pricing, allocation, and purchasing decisions so that prediction and decision optimization are learned in a more unified loop. Third, future studies should examine online deployment issues such as delayed feedback, changing assortment structures, promotion shocks, stochastic lead times, and service-level constraints. These extensions would move PtC from adaptive forecast correction toward a broader closed-loop decision framework for retail supply chains.

\section*{Data Availability Statement}
The data that support the findings of this study are available from the corresponding author upon reasonable request.

\bibliographystyle{chicago}
\spacingset{1}
\bibliography{reference}

\clearpage

\begin{center} 
\textbf{Supplemental Online Materials to “A Predict-then-Correct Loop Based on Few-Shot Continuous Contextual Bandit for Demand
Forecasting"} \end{center}

\appendix

\makeatletter
\@addtoreset{table}{section}
\@addtoreset{figure}{section}
\@addtoreset{algorithm}{section}
\makeatother

\renewcommand{\thetable}{\Alph{section}\arabic{table}}
\renewcommand{\thefigure}{\Alph{section}\arabic{figure}}
\renewcommand{\thealgorithm}{\Alph{section}\arabic{algorithm}}

\newcommand{\appendixsection}[1]{%
    \refstepcounter{section}%
    \section*{Appendix~\thesection:\ #1}%
    \addcontentsline{toc}{section}{Appendix~\thesection: #1}%
}

\appendixsection{Dataset Splitting}\label{app:dataset_splitting}

\begin{table}[htbp]
\centering
\caption{Dataset and splitting}
\label{tab:data_split_settings}
\small
\setlength{\tabcolsep}{3.5pt}
\begin{tabular}{llclccc}
\toprule
Dataset & Cluster & Quantity & Training & Validation & Forecast & Sliding window \\
\midrule
\multirow{3}{*}{Walmart}
& Stable \& High Volume       & 217  & Week 1--134 & Week 135--143 & Week 144--152 & 3 \\
& Stable \& Low Volume        & 3122 & Week 1--134 & Week 135--143 & Week 144--152 & 3 \\
& Erratic \& Intermittent     & 3256 & Week 1--134 & Week 135--143 & Week 144--152 & 3 \\
\midrule
\multirow{3}{*}{Beverage}
& Stable \& High Volume       & 52   & Day 1--1765 & Day 1766--1825 & Day 1826--1885 & 3 \\
& Stable \& Low Volume        & 465  & Day 1--1765 & Day 1766--1825 & Day 1826--1885 & 3 \\
& Erratic \& Intermittent     & 168  & Day 1--1765 & Day 1766--1825 & Day 1826--1885 & 3 \\
\bottomrule
\end{tabular}
\end{table}

\appendixsection{Principal Hyperparameter Settings}
\label{app:hyperparameters}

\begin{table}[htbp]
\centering
\caption{Principal hyperparameter settings}
\label{tab:hyperparameter_settings}
\small
\setlength{\tabcolsep}{8pt}
\begin{tabular}{lc}
\toprule
Setting & Value \\
\midrule
Batch size & 64 \\
Learning rate & $3 \times 10^{-4}$ \\
Optimizer & Adam \\
Hidden dimension & 256 \\
Action dimension & 1 \\
Observation dimension & 15 \\
Entropy coefficient $\alpha$ & 0.02 \\
Network layers & 3 \\
Parameter update ratio for the Walmart dataset & 0.10 \\
Parameter update ratio for the beverage dataset & 0.16 \\
Sliding-window length & 3 \\
\bottomrule
\end{tabular}
\end{table}

\clearpage

\appendixsection{Augmentation Data Selection}
\label{app:augmentation_selection}

\begin{figure}[htbp]
    \centering
    \includegraphics[width=\textwidth]{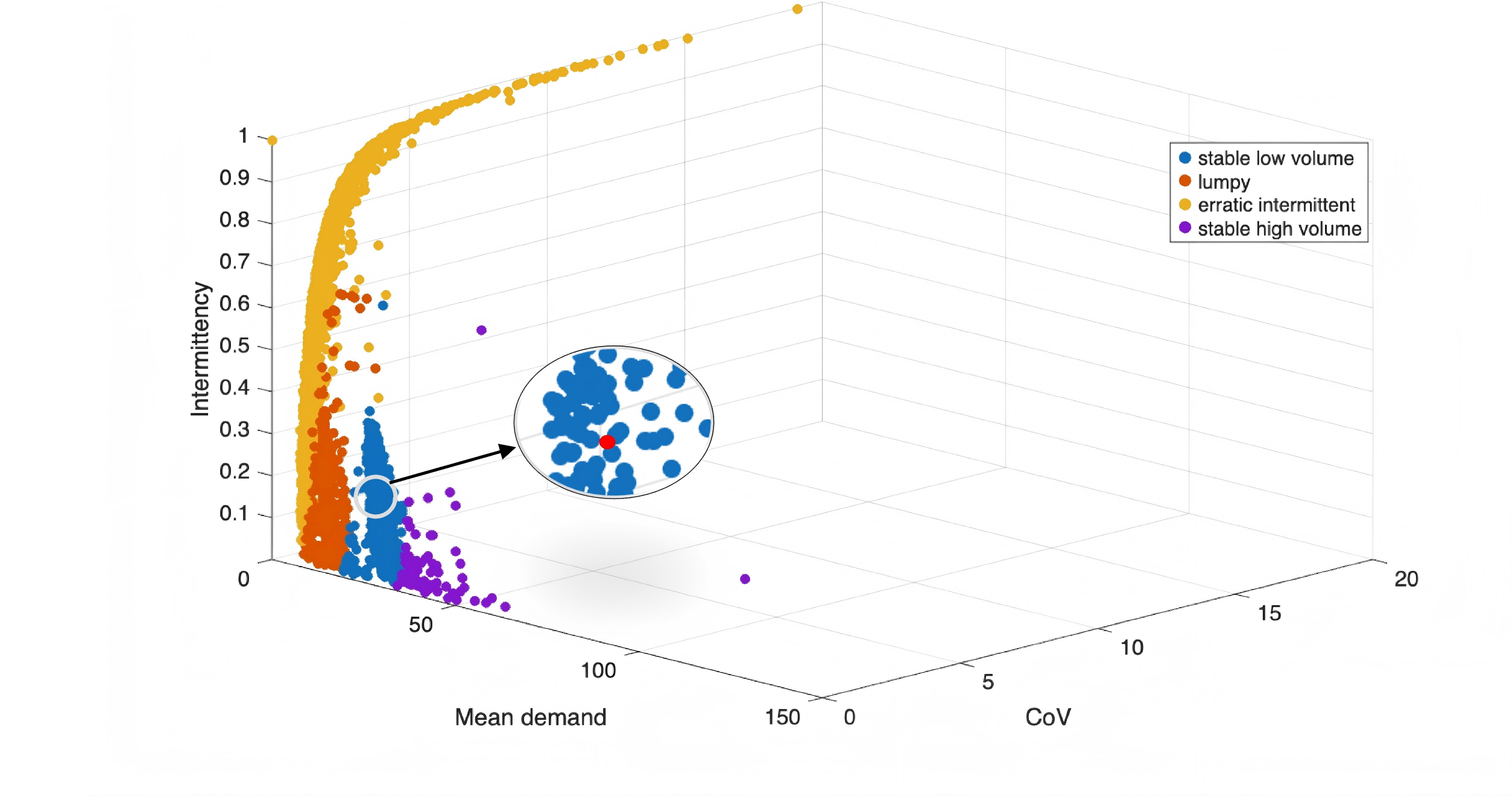}
    \caption{Augmentation data selection within each cluster.}
    \label{fig:cluster}
\end{figure}

\appendixsection{Proof of Proposition 1}
\label{app:proof1}

\begin{proof}[Proof of Proposition~\ref{prop:zero-correction}]
Since $0\in\mathcal{A}$ and $\pi_0(s_t)=0$ belongs to the policy class $\Pi$, the PtC framework can choose the zero-correction policy. Under this policy, the corrected forecast becomes
\[
\hat y_t^{PtC}(\pi_0)
=
\hat y_t^{ML}(1+\pi_0(s_t))
=
\hat y_t^{ML}.
\]
Therefore, the ML-only prediction is a special case of the PtC prediction class. Since the PtC framework optimizes over a feasible policy class that contains the zero-correction policy, taking the infimum over $\Pi$ cannot yield a larger expected loss than the loss attained by $\pi_0$. Hence,
\[
\inf_{\pi\in\Pi}
\mathbb{E}\left[
\frac{1}{T}\sum_{t=1}^{T}
\ell\left(y_t,\hat y_t^{ML}(1+\pi(s_t))\right)
\right]
\le
\mathbb{E}\left[
\frac{1}{T}\sum_{t=1}^{T}
\ell\left(y_t,\hat y_t^{ML}(1+\pi_0(s_t))\right)
\right].
\]
Using $\pi_0(s_t)=0$, the right-hand side reduces to
\[
\mathbb{E}\left[
\frac{1}{T}\sum_{t=1}^{T}
\ell\left(y_t,\hat y_t^{ML}\right)
\right].
\]
\end{proof}

\appendixsection{Proof of Proposition 2}
\label{app:proof2}
\begin{proof}[Proof of Proposition~\ref{prop:top-p-tradeoff}]
Let $\mathcal{L}_{new}(\theta)$ denote the loss on the new demand cycle, and define the feasible parameter space under top-$p$ updating as
\[
\Theta_p=\left\{\theta^0+\Delta:
\operatorname{supp}(\Delta)\subseteq S_p
\right\}.
\]
Then the adaptation gap
\[
G_{ad}(p)
=
\inf_{\theta\in\Theta_p}\mathcal{L}_{new}(\theta)
-
\inf_{\theta\in\Theta_1}\mathcal{L}_{new}(\theta)
\]
is monotone non-increasing with respect to $p$, and satisfies $G_{ad}(1)=0$.

Moreover, suppose that the masked gradient update is performed for $U$ steps:
\[
\theta^{u+1}_p
=
\theta^u_p
-
\eta\, m_p\odot g^u_p,
\quad
u=0,\ldots,U-1,
\]
where $\theta^0_p=\theta^0$, $\eta>0$ is the learning rate, and every gradient coordinate is bounded as $|(g^u_p)_j|\le G$. Then the parameter drift satisfies
\[
\|\theta^U_p-\theta^0\|_2
\le
\eta U G \sqrt{\lceil pd\rceil}.
\]
Therefore, the upper bound on parameter drift is monotone non-decreasing with respect to $p$ and equals zero when $p=0$. If the old-cycle loss
$\mathcal{L}_{old}(\theta)$ is $L_{old}$-Lipschitz continuous, then
\[
\mathcal{L}_{old}(\theta^U_p)-\mathcal{L}_{old}(\theta^0)
\le
L_{old}\eta U G \sqrt{\lceil pd\rceil},
\]
which implies that the forgetting-risk bound is also monotone non-decreasing with respect to $p$.

We first prove the plasticity result. For any $p_1\le p_2$, the nesting property of the Top-$p$ selected sets gives $S_{p_1}\subseteq S_{p_2}$. Therefore, any parameter perturbation supported on $S_{p_1}$ is also supported on $S_{p_2}$, which implies
\[
\Theta_{p_1}\subseteq \Theta_{p_2}.
\]
Since $\Theta_{p_2}$ is a larger feasible parameter space, minimizing the new-cycle loss over $\Theta_{p_2}$ cannot yield a larger optimal value than minimizing over $\Theta_{p_1}$. Hence,
\[
\inf_{\theta\in\Theta_{p_2}}\mathcal{L}_{new}(\theta)
\le
\inf_{\theta\in\Theta_{p_1}}\mathcal{L}_{new}(\theta).
\]
Subtracting the constant full-update benchmark
$\inf_{\theta\in\Theta_1}\mathcal{L}_{new}(\theta)$ from both sides gives
\[
G_{ad}(p_2)\le G_{ad}(p_1).
\]
Thus, the adaptation gap is monotone nonincreasing in $p$. When $p=1$, all parameters are allowed to update, so $\Theta_p=\Theta_1$ and therefore
\[
G_{ad}(1)=0.
\]

We next prove the stability result. By the masked update rule,
\[
\theta^U_p-\theta^0
=
-\eta\sum_{u=0}^{U-1} m_p\odot g^u_p.
\]
Taking the Euclidean norm and applying the triangle inequality yields
\[
\|\theta^U_p-\theta^0\|_2
\le
\eta\sum_{u=0}^{U-1}
\|m_p\odot g^u_p\|_2.
\]
Since the mask $m_p$ contains exactly $\lceil pd\rceil$ nonzero entries and each gradient coordinate is bounded by $G$, we have
\[
\|m_p\odot g^u_p\|_2
\le
G\sqrt{\lceil pd\rceil}.
\]
Therefore,
\[
\|\theta^U_p-\theta^0\|_2
\le
\eta U G\sqrt{\lceil pd\rceil}.
\]
Because $\lceil pd\rceil$ is monotone non-decreasing in $p$, the parameter-drift upper bound is also monotone non-decreasing in $p$. When $p=0$, no parameter is selected for updating, so $m_p=0$ and
\[
\|\theta^U_0-\theta^0\|_2=0.
\]

Finally, if $\mathcal{L}_{old}$ is $L_{old}$-Lipchitz continuous, then
\[
\mathcal{L}_{old}(\theta^U_p)-\mathcal{L}_{old}(\theta^0)
\le
L_{old}\|\theta^U_p-\theta^0\|_2.
\]
Substituting the drift bound gives
\[
\mathcal{L}_{old}(\theta^U_p)-\mathcal{L}_{old}(\theta^0)
\le
L_{old}\eta U G\sqrt{\lceil pd\rceil}.
\]
Thus, increasing $p$ expands the adaptable parameter space and improves plasticity, but it also increases the upper bound on parameter drift and forgetting risk. This establishes the stability-plasticity trade-off of the top-$p$ updating rule.
\end{proof}

\appendixsection{Proof of Proposition 3}\label{app:proof3}
\begin{proof}[Proof of Proposition~\ref{prop:continuous-action-convergence}]
Define the smoothed expected correction objective
\[
\mathcal{J}_{\epsilon}(\theta)
=
\mathbb{E}_{(s_t,y_t,\hat y_t^{ML})\sim\mathcal{D},\,a_t\sim\pi_{\theta}}
\left[
\rho_{\epsilon}\left(\hat y_t^{ML}(1+a_t)-y_t\right)
\right],
\]
where $\rho_{\epsilon}(z)=\sqrt{z^2+\epsilon^2}$. The smoothing parameter $\epsilon>0$ avoids the non-differentiability of the absolute value while
preserving the MAE-type correction objective as $\epsilon\to0$. The notation $a_t\sim\pi_{\theta}$ means that the correction action at time $t$ is generated by the policy parameterized by $\theta$.

At update step $u$, the top-$p$ rule defines a binary mask $m_u\in\{0,1\}^d$. The masked policy update can be written as
\[
\theta^{u+1}
=
\theta^u
-
\eta\,m_u\odot\nabla\mathcal{J}_{\epsilon}(\theta^u).
\]
Let
\[
d_u=m_u\odot\nabla\mathcal{J}_{\epsilon}(\theta^u).
\]
Then $\theta^{u+1}-\theta^u=-\eta d_u$. Since
$\mathcal{J}_{\epsilon}$ is $L$-smooth, the descent lemma gives
\[
\mathcal{J}_{\epsilon}(\theta^{u+1})
\le
\mathcal{J}_{\epsilon}(\theta^u)
+
\nabla\mathcal{J}_{\epsilon}(\theta^u)^{\top}
(\theta^{u+1}-\theta^u)
+
\frac{L}{2}\|\theta^{u+1}-\theta^u\|_2^2.
\]
Substituting $\theta^{u+1}-\theta^u=-\eta d_u$ yields
\[
\mathcal{J}_{\epsilon}(\theta^{u+1})
\le
\mathcal{J}_{\epsilon}(\theta^u)
-
\eta\nabla\mathcal{J}_{\epsilon}(\theta^u)^{\top}d_u
+
\frac{L\eta^2}{2}\|d_u\|_2^2.
\]
Because $m_u$ is a binary mask,
\[
\nabla\mathcal{J}_{\epsilon}(\theta^u)^{\top}d_u
=
\nabla\mathcal{J}_{\epsilon}(\theta^u)^{\top}
\left(m_u\odot\nabla\mathcal{J}_{\epsilon}(\theta^u)\right)
=
\|d_u\|_2^2.
\]
Therefore
\[
\mathcal{J}_{\epsilon}(\theta^{u+1})
\le
\mathcal{J}_{\epsilon}(\theta^u)
-
\eta\left(1-\frac{L\eta}{2}\right)\|d_u\|_2^2.
\]
If $0<\eta\le 1/L$, then $1-L\eta/2\ge 1/2$, and thus
\[
\mathcal{J}_{\epsilon}(\theta^{u+1})
\le
\mathcal{J}_{\epsilon}(\theta^u)-\frac{\eta}{2}\|d_u\|_2^2
\le
\mathcal{J}_{\epsilon}(\theta^u).
\]
This proves the monotonic decrease of the expected correction loss after each
policy update.
\end{proof}

\appendixsection{Top-$p$ Masked Bandit Update}
\label{app:algorithm}

\begin{algorithm}[htbp]
\caption{Top-$p$ masked bandit update for the CB correction policy}
\label{alg:update}
\begin{algorithmic}[1]

\STATE \textbf{Input:} Feedback data $\mathcal{D}_{new}$, initial policy
$\pi_{\theta^0}$, update ratio $p$, learning rate $\eta$,
mini-batch size $B$, and update steps $U$

\STATE \textbf{Output:} Updated correction policy $\pi_{\theta^U}$

\FOR{$u=0,\ldots,U-1$}

    \STATE Sample a mini-batch
    $\mathcal{B}_u \subset \mathcal{D}_{new}$,
    where $|\mathcal{B}_u|=B$

    \STATE For all $t\in\mathcal{B}_u$, compute
    $\mu_{\theta^u}(s_t)$ and $\sigma_{\theta^u}(s_t)$,
    sample
    \[
    \tilde{a}_t
    \sim
    \mathcal{N}
    \left(
    \mu_{\theta^u}(s_t),
    \sigma_{\theta^u}^{2}(s_t)
    \right),
    \]
    and set
    \[
    a_t=\operatorname{clip}(\tilde{a}_t,-1,2)
    \]

    \STATE Compute the corrected forecasts and rewards:
    \[
    r_t
    =
    \tau
    \left(
    |\hat{y}^{ML}_t-y_t|
    -
    |\hat{y}^{CB}_t-y_t|
    \right)
    \]

    \STATE Estimate the bandit policy gradient:
    \[
    \widehat{\nabla}_{\theta}J(\theta^u)
    =
    \frac{1}{|\mathcal{B}_u|}
    \sum_{t\in\mathcal{B}_u}
    (r_t-b_t)
    \nabla_{\theta}
    \log
    \pi_{\theta^u}(a_t\mid s_t)
    \]

    \STATE For each layer $k$, construct mask $M$ by selecting the smallest
    \[
    q_k=\left\lceil p|\theta^k|\right\rceil
    \]
    parameters according to $|\theta_j^k|$

    \STATE Update:
    \[
    \theta^{u+1}
    =
    \theta^u
    +
    \eta
    M
    \odot
    \widehat{\nabla}_{\theta}J(\theta^u)
    \]

\ENDFOR

\STATE \textbf{Return} $\pi_{\theta^U}$

\end{algorithmic}
\end{algorithm}

\end{document}